\documentclass[a4paper,11pt]{article}
\usepackage[a4paper,margin=25mm]{geometry}

\usepackage[utf8]{inputenc}
\usepackage[english]{babel}
\usepackage{amsmath}
\usepackage{amsthm}
\usepackage{amssymb}
\usepackage{thm-restate}
\usepackage{xspace}
\usepackage{algorithm}
\usepackage[noend]{algpseudocode}
\usepackage{url}
\usepackage[medium]{titlesec}

\newtheorem{lemma}{Lemma}
\newtheorem{definition}{Definition}
\newtheorem{corollary}{Corollary}
\newtheorem{observation}{Observation}
\newtheorem{remark}{Remark}

\newcommand{\CEIL}[1]{\left\lceil#1\right\rceil}
\newcommand{\alg}{\ensuremath{\operatorname{\textsc{Alg}}}\xspace}
\newcommand{\opt}{\ensuremath{\operatorname{\textsc{Opt}}}\xspace}
\newcommand{\LFD}{\ensuremath{\operatorname{{LFD}}}\xspace}
\newcommand{\Mark}{\ensuremath{\operatorname{\textsc{Mark}}}\xspace}
\newcommand{\MarkPredict}{\ensuremath{\operatorname{\textsc{Mark\&Predict}}}\xspace}
\newcommand{\Markzero}{\ensuremath{\operatorname{\textsc{Mark0}}}\xspace}
\newcommand{\SET}[1]{\left\{#1\right\}}

\newcommand{\cost}{\zeta}
\newcommand{\page}{r}

\title{Paging with Succinct Predictions\footnote{Boyar, Favrholdt, and Larsen were supported in part by the Independent Research Fund Denmark, Natural Sciences, grant DFF-0135-00018B and in part by the Innovation Fund Denmark,
grant 9142-00001B, Digital Research Centre Denmark,
project P40: Online Algorithms with Predictions.
Polak was supported in part by Swiss National Science Foundation project Lattice Algorithms and Integer Programming (185030).
Part of the research was done during the Workshop on Algorithms with Predictions in the Bernoulli Center for Fundamental Studies at EPFL.}}

\author{Antonios Antoniadis$^1$ \and Joan Boyar$^2$ \and Marek Eliáš$^3$ \and Lene M. Favrholdt$^2$ \and Ruben Hoeksma$^1$ \and Kim S. Larsen$^2$ \and Adam Polak$^4$ \and Bertrand Simon$^5$}

\date{
$^1$ University of Twente, Enschede, Netherlands\\
$^2$ University of Southern Denmark, Odense, Denmark\\
$^3$ Bocconi University, Milan, Italy\\
$^4$ EPFL, Lausanne, Switzerland\\
$^5$ IN2P3 Computing Center and CNRS, Villeurbanne, France
}

\begin{document}

\maketitle

\abstract{Paging is a prototypical problem in the area of online algorithms. It has also played a central role in the development of learning-augmented algorithms -- a recent line of research that aims to ameliorate the shortcomings of classical worst-case analysis by giving algorithms access to predictions. Such predictions can typically be generated using a machine learning approach, but they are inherently imperfect. Previous work on learning-augmented paging has investigated predictions on (i)~when the current page will be requested again (\emph{reoccurrence predictions}), (ii)~the current state of the cache in an optimal algorithm (\emph{state predictions}), (iii)~all requests until the current page gets requested again, and (iv)~the relative order in which pages are requested. 

We study learning-augmented paging from the new perspective of requiring the least possible amount of predicted information. More specifically, the predictions obtained alongside each page request are limited to one bit only. We consider two natural such setups: (i)~\emph{discard predictions}, in which the predicted bit denotes whether or not it is ``safe'' to evict this page, and (ii) \emph{phase predictions}, where the bit denotes whether the current page will be requested in the next phase (for an appropriate partitioning of the input into phases). We develop algorithms for each of the two setups that satisfy all three desirable properties of learning-augmented algorithms -- that is, they are consistent, robust and smooth -- despite being limited to a one-bit prediction per request. We also present lower bounds establishing that our algorithms are essentially best possible.}

\section{Introduction}

\emph{Paging} (also known as \emph{caching}) is a classical online problem, and an important special case of several other online problems~\cite{BE98}, which can be motivated through resource management in operating systems. You are given a fast cache memory with capacity to simultaneously store at most a constant number, $k$, of pages. Requested pages, according to a sequence of \emph{page requests}, have to be loaded into the cache to be served by the operating system. More specifically, pages are requested one by one in an online fashion, and each request needs to be immediately served upon its arrival. Serving a page is done at zero cost if the requested page currently resides in the cache. If this is not the case, then a \emph{page fault} occurs and the page has to first be loaded into the cache (after potentially evicting some other page to make space). This incurs a fixed cost. The underlying algorithmic question is to decide which page to evict each time a page has to be loaded into the cache, with the goal to minimize the total incurred cost, i.e., the total number of page faults.

Paging has been extensively studied and is well-understood. There exists an optimal offline algorithm, $\LFD$ (\emph{longest forward distance}), that simply follows the so-called \emph{Belady's rule:} always evict the page that will be requested again the furthest in the future. Note that Belady's rule can only be applied to the offline variant of the problem, where all future page requests are known to the algorithm. With respect to online algorithms, no deterministic online algorithm can obtain a competitive ratio\footnote{Competitive ratio is the standard performance metric for online algorithms, see Section~\ref{sec:contrib} for a definition.} smaller than $k$~\cite{ST85}. Two simple deterministic algorithms that are $k$-competitive~\cite{ST85} exist: FIFO (evict the oldest page in the cache) and LRU (evict the least recently used/requested page from the cache). Fiat et al.~\cite{FKLMSY91} developed a randomized algorithm called \emph{\Mark} that is $(2H_k -1)$-competitive\footnote{$H_k=\sum_{i=1}^k1/i$ is the $k$'th harmonic number. Recall that $\ln k \leq H_k\leq 1+ \ln k$.}~\cite{ACN00}. Furthermore this is tight, up to a constant factor of $2$, since no randomized algorithm can obtain a competitive ratio better than $H_k$~\cite{FKLMSY91}. Later, optimal $H_k$-competitive randomized algorithms were discovered~\cite{ACN00,MS91}.

The above results are tight in the worst case,  although inputs encountered in  many practical situations  may allow for a better performance. The novel research area of \emph{learning-augmented algorithms} attempts to take advantage of such opportunities and ameliorate shortcomings of worst-case analysis by assuming that the algorithm has black-box access to a set of (e.g., machine-learned) predictions regarding the input. Naturally, the quality of these predictions is not known a priori, hence the goal is to design algorithms with a good performance on the following parameters:
{\em robustness}, which is the worst-case performance guarantee that holds independently of the prediction accuracy;
{\em consistency}, which is the competitive ratio under perfect predictions;
and {\em smoothness}, which describes the rate at which  the competitive ratio deteriorates with increasing prediction error.

Given the central role of paging within online algorithms, it is no surprise that learning-augmented paging has been extensively studied as well, and actually a significant number of papers in the area are either directly or indirectly linked to the paging problem. Examples include the seminal paper by Lykouris and Vassilvitskii~\cite{LV21}, who studied \emph{reoccurence predictions}, i.e., along with each page request the algorithm obtains a prediction on the timepoint of the next request of that page. Their results were later refined by Rohatgi~\cite{R20} and Wei~\cite{W20}. Jiang et al.~\cite{JiangP020} investigated the setting in which all requests until the next request of the currently requested page are predicted, whereas Bansal et al.~\cite{BansalCKPV22} considered predictions regarding the relative order in which the pages are requested. Antoniadis et al.~\cite{ACE0S20} looked into so-called \emph{state predictions} that predict the cache-contents of an optimal algorithm.

Although the above  algorithms have been analyzed with respect to their consistency, robustness and smoothness, no consideration has been made regarding the total amount of predicted information. Given that the predicted information needs to be computed through a separate black-box algorithm and also communicated to the actual paging algorithm for each request, a learning-augmented algorithm that is based on a large amount of predicted information may be impractical in a real-world application.

In this paper we study learning-augmented paging while taking a new approach, requiring a minimal amount of predicted information. We assume that the predictions must be encoded in only one bit per request. This is indeed the least possible amount of predicted information (up to a constant) since any (deterministic or randomized) algorithm that receives perfect predictions that can be encoded in sublinearly many bits (in the length of the request sequence) cannot be better than $H_k$-competitive~\cite{M16}. Moreover, there are binary classifiers producing one-bit predictions for paging~\cite{JainL16,ShiHJL19} which have great performance in practice (see Section~\ref{sec:related-work} for more details) and it is desirable to use them in learning-augmented algorithms.

We study two natural such {\emph{setups},} with one-bit predictions, which we call \emph{discard predictions} and \emph{phase predictions}. The predicted bit in discard predictions denotes whether LFD would evict the current page before it gets requested again. 
In phase predictions, the bit denotes whether the current page will be requested again in the following $k$-phase (the notion of a $k$-phase is based on marking algorithms, such as \Mark and LRU, and it is formally defined in Section~\ref{sec:prelim}). Both of these new setups can be interpreted as condensing the relevant information from the previously existing setups into one bit per request.

We develop algorithms for each of the two setups that satisfy all three desirable properties of learning-augmented algorithms -- that is, they are consistent, robust and smooth -- despite being limited to a one-bit prediction per request. We also present lower bounds establishing that our algorithms are essentially best possible.

\subsection{Our contribution}
\label{sec:contrib}

An important preliminary observation is that there is an asymmetry regarding prediction errors:
Wrongly evicting a page will generally only lead to one page-fault once that page is requested again, however keeping a page which should be evicted in the cache can lead to multiple page-faults while the algorithm keeps evicting pages that will be requested again soon.
For this reason we distinguish between two types of prediction errors. For a sequence of $n$ page requests, let $p\in\{0,1\}^n$ be the vector of predictions and $p^*\in\{0,1\}^n$ be the ground truth, where, intuitively, a value of $0$ means (in both setups) that, according to the prediction, the page requested should stay in cache.
We define $\eta_0$ and $\eta_1$ as the numbers of incorrect
predictions $0$ and $1$, respectively, usually leaving out the parameters $p$ and $p^*$ when they are understood:

\[\eta_h(p,p^*) = | \{ i \in [n] \mid p_i = h, p^*_i = 1-h\}|\,, \quad \text{for } h\in\{0,1\}.\]

In order to capture how different types of errors affect the cost of an algorithm, we generalize the notion of competitive ratio to what we call \emph{$(\alpha,\beta,\gamma)$-competitiveness}.

\begin{definition}
A learning-augmented online paging algorithm \alg is called $(\alpha,\beta,\gamma)$-com\-pe\-ti\-tive if there exists a constant $b$ (possibly depending on $k$) such that for any instance $I$ with ground truth $p^*$ and any predictions $p$,
\[
    \alg(I,p) \leq \alpha \cdot \opt(I) + \beta \cdot \eta_0(p,p^*) + \gamma \cdot \eta_1(p,p^*) + b\,,
\]
where $\alg(I,p)$ and $\opt(I)$ denote\footnote{Following a standard practice in online algorithms literature, in what follows, we abuse the notation and use \alg and \opt to denote both the algorithms and their respective costs incurred on the implicitly understood instance that we are currently reasoning about.} costs incurred on this instance by
the online algorithm and the offline optimal algorithm, respectively,
and $\eta_0, \eta_1$ denote the two types of error of predictions provided
to the online algorithm.
\end{definition}

{Note that the notion of $(\alpha,\beta,\gamma)$-competitiveness generalizes that of the (classical) competitive ratio: an algorithm is $c$-competitive if and only if it is $(c,0,0)$-competitive.} Furthermore, it is easy to see that $(\alpha,\beta,\gamma)$-competitiveness directly implies a consistency of $\alpha$; it also quantifies the smoothness of the algorithm. We can achieve robustness as follows: any deterministic $(\alpha,\beta,\gamma)$-competitive algorithm for paging can be combined with LRU or FIFO through the result of Fiat et al.~\cite{FKLMSY91} to give a deterministic algorithm with a consistency of $(1+\epsilon)\alpha$ and a robustness\footnote{Actually, Fiat et al.~\cite{FKLMSY91} show the more general result that one can combine $m$  algorithms such that for any input instance $I$ this combination incurs a cost that is within a factor $c_i$ from the cost of each corresponding algorithm $i$ on $I$. The constants $c_i$ can be chosen arbitrarily as long as they satisfy $\sum_{i=1}^m 1/c_i \le 1$.} of $\frac{1+\epsilon}{\epsilon} k$, for any $\epsilon > 0$. Similarly, any randomized $(\alpha,\beta,\gamma)$-competitive algorithm for paging can be combined (see~\cite{ACE0S20} and~\cite{BB00}) with an $H_k$-competitive algorithm~\cite{ACN00,MS91} to give a $((1+\epsilon )\alpha)$-consistent and $((1+\epsilon )H_k)$-robust algorithm.
Both of these combination approaches work independently of the considered prediction setup. We therefore focus the rest of the paper on giving upper and lower bounds for the $(\alpha,\beta,\gamma)$-competitiveness.

As explained at the beginning of this section, the two types of prediction errors have significantly different impact: keeping a page in cache while it was safe to evict is potentially much more costly than evicting a page that should have been kept. Hence, $\beta$ will intuitively be much larger than $\gamma$ in our results. Our lower bounds also show that $\alpha+\beta$ {cannot be} smaller than the best classical competitive ratio.

We remark that previous papers on learning-augmented paging (e.g., \cite{LV21,R20,W20})
analyze smoothness by expressing the (classical) competitive ratio as a function of the normalized prediction error $\frac{\eta}{\opt}$, and that our results could also be stated in that manner because every $(\alpha, \beta, \gamma)$-competitive algorithm is also $(\alpha + \beta \cdot \frac{\eta_0}{\opt} + \gamma \cdot \frac{\eta_1}{\opt})$-competitive in the classical sense.

\paragraph{Discard-predictions setup upper bounds.} In Section~\ref{sec:ub-discard} we develop a deterministic and a randomized algorithm for the discard-predictions setup:

\begin{restatable}{theorem}{ThmDiscardDeterministicUb}
\label{thm:discard-deterministic-ub}
There is a deterministic~$(1,k-1,1)$-competitive algorithm for the discard-predictions setup.
\end{restatable}

The algorithm realizing Theorem~\ref{thm:discard-deterministic-ub} is very simple and natural: On each page-fault, evict a page that is predicted as safe to evict, if such a page exists. If it does not exist, then just flush the cache, i.e., evict all pages that it contains. The analysis is based on deriving appropriate bounds on the page-faults for both \alg and \opt within any two consecutive flushes, as well as the respective prediction error.

\begin{restatable}{theorem}{ThmDiscardRandomizedUb}
\label{thm:discard-randomized-ub}
There is a randomized $(1,2H_k, 1)$-competitive algorithm for the discard-predictions setup.
\end{restatable}

Compared to the deterministic algorithm above,
the algorithm from Theorem~\ref{thm:discard-randomized-ub}
uses an approach resembling the classical \Mark algorithm
when evicting pages predicted 0.
However, we note that it does not fall into the class of so-called \emph{marking algorithms} (see Section~\ref{sec:prelim}),
as pages predicted 1 are evicted sooner. This is essential
for achieving $\alpha=1$ but requires a different definition
of phases and a novel way of charging evictions of pages predicted 0.

\paragraph{Phase-predictions setup upper bounds.}
For phase-predictions, in Section~\ref{sec:ub-phase} we design an algorithm called \MarkPredict
which can be seen as a modification of the classical \Mark algorithm
giving priority to pages predicted 1 when choosing a page to evict.
We prove two bounds for this algorithm.
The first one is sharper for small $\eta_1$
and, in fact, it holds even with deterministic evictions of
pages predicted 1.

\begin{restatable}{theorem}{ThmMarkRandA}
\label{thm:markrand1}
\MarkPredict is a randomized~$(2,H_k,1)$-competitive algorithm for the phase-predictions setup.
\end{restatable}

The second bound exploits the random eviction of pages predicted 1
and gives a much stronger result if $\eta_1$ is relatively large.

\begin{restatable}{theorem}{ThmMarkRandB}
\label{thm:markrand2}
\MarkPredict is a randomized $\big(2, H_k, \gamma(\eta_1/\opt)\big)$-competitive algorithm for the phase-predictions setup, where
\[ \gamma(x) = 2x^{-1} \left( \ln (2x+1) + 1\right). \]
\end{restatable}
In other words, the (expected) cost of \MarkPredict is at most
\[2\left(\ln\left(\frac{2\eta_1}{\opt}+1\right) + 2\right) \cdot \opt + H_k \cdot \eta_0.\]

Note that this expression should not be considered when $\eta_1\leq\opt$ as $\gamma(1)>1$ so the guarantee of the previous theorem would then be stronger. For $\eta_1>\opt$, multiple possibilities exist to phrase the above expression into our $(\alpha,\beta,\gamma)$-competitiveness notion, so we chose the one matching the previously established value of $\alpha$.
To illustrate the gain over the previous bound, with $\eta_1/\opt = \Omega(k)$, we obtain $\gamma(\eta_1/\opt) = O\left(\frac{\log k}{k}\right)$, thus matching the lower bound of Theorem~\ref{thm:randlowerbound}.

\paragraph{Lower bounds.}
In Section~\ref{sec:lower-bounds}, we give lower bounds that show that the upper bounds above are essentially tight. More specifically, we prove the following for the two considered setups.

\begin{restatable}{theorem}{ThmDeterministicLowerBound}
    \label{thm:deterministiclowerbound}
    In both the discard-predictions and phase-predictions setups, there is no deterministic $(\alpha,\beta,\gamma)$-competitive  algorithm such that
    either $\alpha+\beta < k$ or $\alpha+(k-1)\cdot\gamma < k$.
\end{restatable}

This directly implies that if $\alpha$ is a constant independent of $k$, then $\beta=\Omega(k)$ and $\gamma=\Omega(1)$. A special case is that any $1$-consistent deterministic algorithm must have $\beta$ at least $k-1$ and $\gamma$ at least 1, matching the upper bound of Theorem~\ref{thm:discard-deterministic-ub}, or, more precisely:

\begin{corollary}
\label{cor:detlower}
In both setups, no deterministic paging algorithm is $(1,k-1-\epsilon,\gamma)$- or $(1,\beta, 1-\epsilon)$-competitive, for any constant $\epsilon>0$ and any value of $\beta$ and $\gamma$. 
\end{corollary}

An analogous lower bound can be obtained for randomized algorithms as well.

\begin{restatable}{theorem}{ThmRandLowerBound}
\label{thm:randlowerbound}
In both the discard-predictions and phase-predictions setups, there is no $(\alpha,\beta,\gamma)$-competitive randomized algorithm such that either $\alpha+\beta < H_{k}$ or $\alpha+(k-1)\cdot\gamma < H_k$, where $H_i= \ln i + O(1)$ is the $i$-th harmonic number.
\end{restatable}

This result implies that, in the upper bounds of Theorems~\ref{thm:markrand1} and \ref{thm:markrand2} for \MarkPredict, the value of $\beta$ is tight up to an additive term of $2$ and the asymptotic value of $\gamma$, when $\eta_1/\opt$ is large, cannot be improved by more than a constant factor in Theorem~\ref{thm:markrand2}.

\begin{corollary}
In both setups, no randomized paging algorithm is $(2,H_k-2-\epsilon,\gamma)$- or $(2,\beta, \frac{H_k-2}{k-1}-\epsilon)$-competitive, for any constant $\epsilon>0$ and any value of $\beta$ and $\gamma$. 
\end{corollary}

The previous theorem implies a subconstant lower bound ($\approx \frac1k\log k$) on $\gamma$ for a value of $\alpha$ up to $O(\log k)$. We complement it by showing that $\gamma$ is lower bounded by a constant if we want to achieve $\alpha=1$.

\begin{restatable}{theorem}{ThmRandLowerBoundB}
\label{thm:randlowerbound2}
There is no $(1,\beta,\gamma)$-competitive randomized algorithm such that
$\gamma < 1/7$ for the discard-predictions setup or $\gamma<1/2$ for the phase-predictions setup.
\end{restatable}

The last two theorems imply that, in the upper bound of Theorem~\ref{thm:discard-randomized-ub}, the values of $\beta$ and $\gamma$ cannot be improved by more than a constant factor.

\begin{corollary}
In the discard-predictions setup, no randomized paging algorithm is $(1,H_k-1-\epsilon,\gamma)$- or $(1,\beta, \frac 17-\epsilon)$-competitive, for any constant $\epsilon>0$ and any value of $\beta$ and $\gamma$. 
\end{corollary}

Similarly to the lower bounds known for classical paging, all three of our lower bound results are based on instances coming from a universe of $k+1$ many pages. However, in order to achieve the desired bounds we need to carefully define the prediction sequence. Somewhat surprisingly, in each of our lower bound results, we are able to use the same prediction sequence for both prediction setups.

\subsection{Further related work}
\label{sec:related-work}

\paragraph{Paging with few predictions.}
In a very recent paper~\cite{Im0PP22}, Im et al.\ consider a different approach to limiting the amount of predicted information within learning-augmented paging. The algorithm has access to an ML-oracle which can be at any time queried about the reoccurrence prediction for any page in the cache. They analyze the trade-offs between the number of queries, the prediction error and algorithm performance. Furthermore, the competitive ratio of the obtained algorithms is $O(\log_{b+1} k)$, where $b$ is the number of queries per page fault. Thus, the consistency of the algorithm would generally be quite far from those of the algorithms presented in this paper.

\paragraph{Other learning-augmented online algorithms.}
In addition to the already mentioned results on learning-augmented paging, several exciting learning-augmented algorithms have been developed for various online problems, including among others weighted paging~\cite{BansalCKPV22}, k-server~\cite{LindermayrMS22}, metrical task systems~\cite{ACE0S20}, ski-rental~\cite{PurohitSK18,ACEPS21}, non-clairvoyant scheduling~\cite{PurohitSK18,LindermayrM22}, online-knapsack~\cite{ImKQP21,Zeynali0HW21,BoyarFL22}, secretary and matching problems~\cite{DuttingLLV21,AntoniadisGKK20}, graph exploration~\cite{EberleLMNS22}, as well as energy-efficient scheduling~\cite{BamasMRS20,AntoniadisGS22,ACEPS21}. Machine- learned predictions have also been considered for designing offline algorithms with an improved running time, see for instance the results of
Dinitz et al.~\cite{DinitzILMV21} on matchings,
Chen et al.~\cite{chen22v} on graph algorithms,
Ergun et al.~\cite{ErgunFSWZ22} on $k$-means clustering, 
Sakaue and Oki~\cite{SO22} on discrete optimization, and
Polak and Zub~\cite{PZ22} on maximum flows. An extensive list of results in the area can be found on~\cite{website}. We would also like to point the reader to the surveys~\cite{MV20,MV22} by Mitzenmacher and Vassilvitskii.

We note that, although our work is closer in spirit to the aforementioned results on learning-augmented paging, our notion of $(\alpha,\beta,\gamma)$-competitiveness is an extension of the $(\rho,\mu)$-competitiveness from~\cite{ACEPS21}. While $(\rho,\mu)$-competitiveness  captures the tradeoff between the dependence on the optimal cost and the prediction error,  $(\alpha,\beta,\gamma)$-competitiveness captures the three-way tradeoff between the dependence on the optimal cost and the two kinds of prediction errors.

\paragraph{Advice complexity.}
An inspiration for considering paging with succinct predictions is that ideas from the research area of \emph{advice complexity} could possibly be applied to learning-augmented algorithms; in particular, the advice from~\cite{DKP09} for the paging problem. The goal, when studying online algorithms with advice, is to determine for online problems how much information about the future is necessary and sufficient to perform optimally or to achieve a certain competitive ratio. This is formalized in different computational models, all of which assume that the online algorithm is given some number of bits of advice~\cite{DKP09,HKK10,BKKKM17,EFKR11}. (See the survey on online algorithms with advice~\cite{BFKLM17}.) The difference from learning-augmented algorithms is that the advice is always correct, so robustness is not a consideration, and the emphasis is on the number of bits the algorithms use, rather than if one could realistically expect that the advice could be obtained.
The advice-complexity result that is probably the closest to our work is by Dobrev et al.~\cite{DKP09} who studied advice that is equivalent to the ground truth for our discard predictions. Their result implies for our setting that when the predictions are guaranteed to be perfect (as one assumes in advice complexity), 
then one can obtain a simple $1$-competitive algorithm, with predictions of just one bit per request. However, it does not immediately imply a positive result in our setting when the predictions are of unknown quality.

\paragraph{Discard predictions in practice.}
Previous research suggests the practicality of the succinct predictions presented in this paper.
Jain and Lin~\cite{JainL16} proposed Hawkeye, an SVM-based binary classifier whose goal is to predict whether a requested page is likely to be kept in cache by the optimal Belady's algorithm. The classifier labels each page as either \emph{cache-friendly} or \emph{cache-averse}, which directly correspond to zero and one, respectively, in our discard-prediction setup. Hawkeye's predictions were accurate enough for wining the 2nd Cache Replacement Championship. Later, Hawkeye was outperformed by Shi et al.'s Glider~\cite{ShiHJL19}, a deep learning LSTM-based predictor that solves the same binary classification problem.
On the other hand, machine-learning models capable of producing reoccurrence predictions and state predictions only recently started being developed, and, while they also have a surprisingly high accuracy, they are prohibitively large and slow to evaluate for performance-critical applications~\cite{LiuHSRA20}.

\subsection{Open problems}
\paragraph{Better dependence on $\eta_1$ in discard-predictions setup.}
For the case of large $\eta_1$, we provide a stronger guarantee for \MarkPredict in
Theorem~\ref{thm:markrand2}.
However, we were not able to obtain a comparable result for the discard-predictions setup, and it would be interesting to further close the gap for the case of large $\eta_1$ as well.
Somewhat surprisingly, an important challenge towards that direction seems to be that of recognizing the presence of an incorrect $0$-prediction early enough.
This can be easily done in the phase-predictions setup;  and we do actually properly account for all incorrect
$0$-predictions (see Observation~\ref{obs:phase-error}).
On the other hand, our criterion in the discard-predictions setup
(see Observation~\ref{obs:discard-error}) may overlook some of them.
This in turn may lead an algorithm to keep the cache full with pages associated with $0$-predictions, forcing it to evict all pages with $1$-predictions, implying $\gamma\geq 1$.

\paragraph{Other online problems with succinct predictions.} For many online problems, the possibility of obtaining good succinct predictions might be more realistic than obtaining more precise, lengthy predictions. It would be interesting to see if such predictions still allow for effective learning-augmented algorithms. Prior results on advice complexity give meaningful lower bounds with respect to the size of such predictions and may provide guidance on what to predict.

\section{Preliminaries}
\label{sec:prelim}

\paragraph{Classical paging.}
In paging, we have a (potentially large) universe~$U$ of pages and a~cache of size~$k$.
At each time step~$i=1, \dotsc, n$, we receive a request~$r_i$ {to a page in} $U$ which needs to be
satisfied by  {loading the page associated to}~$r_i$ to the cache {(if it is not in the cache already)}. {This may require} evicting some other page to make
space for the requested page. The {goal of an algorithm is to serve the whole request sequence at minimal \emph{cost}}. The cost of an algorithm is the number of page loads (and therefore also the number of page faults)
performed to serve the request sequence.
Note that this number is within an additive term~$k$ from the number of page evictions.
In our analyses, we can choose to work with whichever of these two quantities is easier to estimate, because of the additive constant in the definition of competitiveness.

When making space for {the page associated to}~$r_i$, \emph{online} algorithms have to decide which page to evict
without knowledge of~$r_{i+1}, \dotsc, r_n$, while \emph{offline} algorithms have this information.

\paragraph{Marking algorithms.}
For the purpose of designing marking algorithms, 
we partition the request sequence into~\emph{$k$-phases}.
A~$k$-phase is a maximal subsequence of at most~$k$ distinct pages.
The first~$k$-phase starts at the first request, and any subsequent $k$-phase $i$ starts at the first request following the last request of~$k$-phase~$i-1$.

The following automatic procedure helps designing algorithms for caching:
at the beginning of each~$k$-phase, we unmark all pages. Whenever a page is requested for the
first time in a~$k$-phase, we mark it. We say that an algorithm belongs to the class
of \emph{marking algorithms}, if it never evicts a marked page.
All marking algorithms are (at most) $k$-competitive \cite{T98} and they have the same cache content
at the end of each~$k$-phase: the~$k$ marked pages which were requested during that~$k$-phase.

Algorithm \Mark~\cite{FKLMSY91} evicts an unmarked page chosen uniformly at random.
In the~$i$th~$k$-phase, with~$c_i$ pages requested that were not requested in~$k$-phase~$i-1$
(we call such pages {\em new}, the others are called {\em old}),
it has in expectation $\sum_{j=1}^{k-c_i}\frac{c_i}{k-(j-1)}\leq c_i(H_k - H_{c_i} +1)$ page faults.
One can show that~$\opt \geq \frac12\sum_{i=1}^m c_i$, where~$m$ is the total number of~$k$-phases in the request sequence, 
and hence \Mark is at most~$2H_k$-competitive. We refer to \cite{BE98} for more details.

\paragraph{Receiving predictions.}
Each request~$r_i$ comes with a \emph{prediction},~$p_i\in\SET{0,1}$.
If a request comes with a prediction of~$0$ (resp.~$1$), we call it a~$0$-prediction (resp.~$1$-prediction), and the requested page a~$0$-page (resp.~$1$-page) until the next time it is requested.
Throughout the paper, we use $r_i$ to refer both to the request and to the page associated with that request.

\paragraph{Discard-predictions setup.}

In this setup, we fix an optimal offline algorithm, say \LFD.
When a requested page is not in cache, \LFD evicts any page that will never be
requested again, if such a page exists, and otherwise evicts the unique
page of the~$k$~pages in cache that will be requested again furthest out
in the future.

Prediction $p_i$ for request $r_i$ is supposed to predict the ground truth $p_i^*$ defined as:
\[p_i^* = \begin{cases}
   0, & \text{if \LFD keeps~$r_i$ in cache until it is requested again,} \\
   1, & \text{if \LFD evicts~$r_i$ before it is requested again.}
\end{cases}\]
For a page~$r_i$ that \LFD retains in cache until the end of the request sequence, $p_i^*=0$.

For simplicity, we define $p^*$ with respect to a fixed optimal algorithm. However, if the prediction vector $p$ happens to predict well the behavior of any other (good but not necessarily optimal) algorithm, then our upper bounds hold also with respect to the performance of that algorithm in place of \opt.

\paragraph{Phase-predictions setup.}
In this setup, we partition the request sequence into $k$-phases, as described above in the paragraph on marking algorithms.

We define the ground truth $p_i^*$ for request~$r_i$ in some $k$-phase~$j$ as follows:

\[p_i^* = \begin{cases}
   0, & \text{if $r_i$ is requested in $k$-phase~$j+1$,} \\
   1, & \text{if $r_i$ is not requested in $k$-phase~$j+1$.}
\end{cases}\]

Note that, in both setups, at the point where a decision is made as to which page to evict, the algorithms only consider the most recent prediction for each page, the one from the most recent request to the page. In the discard-predictions setup, this is the only logical possibility. In the phase-predictions setup, there could theoretically be a page,~$p$, requested more than once in phase~$i$, where the prediction (as to whether or not it will be requested in phase~$i+1$) is inconsistent within phase~$i$. We assume that the last prediction is the most relevant, so only {this one} is used by our algorithms, and only this one contribute towards a possible error in~$\eta_0$ or $\eta_1$.
In fact, if convenient for implementation, we could avoid running
the predictor at repeated requests by producing predictions at once for each
page in the cache at the end of phase~$i$.
In addition, in the phase-predictions setup, predictions in the last phase do not count at all, and in particular, do not count in~$\eta_h$.

In a given phase, the pages that are in cache at the beginning of the phase are called {\em old} pages.
Pages requested within a phase that are not old are called {\em new} pages.
Thus, all requests in the first phase are to new pages.

\section{Algorithms with discard-predictions}
\label{sec:ub-discard}

{We first investigate the discard-predictions setup.}
The following simple observation is useful in the analyses of our  algorithms.
\begin{observation}
\label{obs:discard-error}
Consider a moment
when there is a set~$S$ of $0$-pages 
(whose most recent prediction is~$0$) 
of size~$k+c-1$ and a page~$r\notin S$ is
requested. Then, at least~$c$ pages from~$S$ have incorrect prediction.
\end{observation}
\begin{proof}
Page~$r$ surely has to be in cache.
Each~$\rho \in S$ has prediction~$0$ by the definition of~$S$. 
Since the cache has size~$k$, any algorithm needs to have evicted at least~$1+|S| - k= c$ pages from~$S$. 
In particular this is true for \LFD. Therefore at least~$c$ pages from~$S$ have incorrect predictions.
\end{proof}

\subsection{Deterministic algorithm}

Our first algorithm is deterministic, and, despite being very simple, it attains the best possible $(\alpha,\beta,\gamma)$-competitiveness for 1-consistent deterministic algorithms (see the lower bound in Theorem~\ref{thm:deterministiclowerbound}).

\ThmDiscardDeterministicUb*

\begin{proof}
Consider the deterministic algorithm, \alg, that on a fault evicts an arbitrary $1$-page, if there is such a page in cache, and flushes the cache otherwise.

We count evictions, and note that up to an additive constant (depending on~$k$), this is the same as the number of faults. 
We divide the request sequence into \emph{stages}, starting a new stage when \alg flushes the cache (i.e., when it is full and contains only $0$-pages). 
We assume an integer number of stages (an assumption that also only adds up to an additive constant; again, depending on~$k$) and consider one stage at a time.

First consider $0$-pages that are evicted. {By definition}, \alg evicts~$k$ such pages in the stage. Since~$k$ $0$-pages have arrived in the stage, and a new page must arrive for \alg to flush,
at least one of the $0$-pages has an incorrect prediction
(obvious here, but captured more generally by Observation~\ref{obs:discard-error})
and \opt must have evicted at least one of these $k+1$ pages.

Letting a superscript,~$s$, denote the values of just this stage, and a subscript denote $0$-pages and $1$-pages, respectively, since both~$\opt^s_0$ and~$\eta^s_0$ are at least one,~$\alg^s_0 \leq \opt^s_0 + (k-1)\eta^s_0$.

Considering $1$-pages, \alg clearly obtains the same result as \opt, except when there is a misprediction, which adds a cost of~$1$. Thus,~$\alg^s_1 \leq \opt^s_1 + \eta^s_1$.

Summing over both predictions and all stages,~$\alg \leq \opt + (k-1)\eta_0 + \eta_1$.
\end{proof}

\begin{remark}
In Theorem~\ref{thm:discard-deterministic-ub}, the choices $\alpha=1$ and $\beta=k-1$ can be generalized, showing that the algorithm is $(\alpha,k-\alpha,1)$-competitive, for $1\leq \alpha \leq k$ (compare with Theorem~\ref{thm:deterministiclowerbound}).
\end{remark}

\subsection{Randomized algorithm}

Now, we present \Markzero, a randomized algorithm that evicts all~$1$-pages immediately. Therefore, whenever the cache is
full and eviction is needed, all the pages in the cache must be~$0$-pages and this situation
signals a presence of an incorrect $0$-prediction.
Since we cannot know which~$0$-page has incorrect prediction, we evict a random unmarked
one in order to make sure that such evictions can be charged to $\eta_0$ in the analysis.
\Markzero is described in Algorithm~\ref{alg:mark-0}.

\begin{algorithm}
  \begin{algorithmic}[1]
    \State $S := \emptyset$
    \State evict all~$1$-pages
    \For {$i = 1$ \textbf{to}~$n$}
      \If {$\page_i$ is not in cache}
        \If {cache is full and all pages from~$S$ in cache are marked}
          \State $S:=$ current cache content
          \State unmark all pages
          \label{alg3:new-phase}
        \EndIf
        \If {$\page_i \in S$ is unmarked and cache contains some unmarked page from~$S$}
          \State evict an unmarked page from~$S$ chosen uniformly at random
          \label{alg3:S-replace}\\
          \algorithmiccomment{We perform this eviction even if the cache is not full}
          
        \EndIf
        \If {cache is full}
          \State evict an unmarked page from~$S$ chosen uniformly at random
          \label{alg3:new-error}
        \EndIf
        \State bring~$\page_i$ to cache
      \EndIf
      \State mark~$\page_i$
      \If {$p_i = 1$}
        \State evict~$\page_i$
        \label{alg3:evict-1-page}
      \EndIf
    \EndFor
\end{algorithmic}
  \caption{\Markzero Eviction Strategy}
  \label{alg:mark-0}
\end{algorithm}

Before proving the competitive ratio of \Markzero, we state a few observations, starting by a simple bound on the evictions of $1$-pages.

\begin{observation}
\label{obs:alg3-1evict}
The number of $1$-pages that \Markzero evicts is at most $\opt + \eta_1$.
\end{observation}

Therefore it is enough to count evictions of 0-pages.
We call a period between two executions of line~\ref{alg3:new-phase} a phase. Phases are similar to $k$-phases in marking algorithms, with the difference being that $1$-pages are directly evicted by the algorithm, even though such a page is still marked.
Phase 1 starts the first time that the cache is full
and a page-fault occurs (recall that this implies that there has been an incorrect prediction on a $0$-page), since~$S=\emptyset$ and the condition on all pages from~$S$ in cache
being marked is vacuously true.
We define phase 0 to be the time from the beginning of the request sequence
until the start of Phase 1.

The following observation bounds the number of evictions of~$0$-pages based on the number of times an eviction is caused by a full cache. An eviction caused by a full cache leads to an unmarked page from~$S$ being evicted. A classical probabilistic argument is then used to bound the number of times a randomly evicted unmarked page is requested again in the phase.

\begin{observation}
\label{obs:alg3-0evict}
Consider a phase with~$c$ executions of line~\ref{alg3:new-error}.
The expected number of evictions of~$0$-pages is at most
$c H_k$.
\end{observation}
\begin{proof}
Note that the number of evictions of $0$-pages is equal to the number of evictions 
in lines~\ref{alg3:new-error} and~\ref{alg3:S-replace}.
There are~$c$ evictions made in line~\ref{alg3:new-error} and we just need to count
evictions made in line~\ref{alg3:S-replace}.

In each execution of line~\ref{alg3:new-error}, we evict a page from~$S$. 
In each execution of line~\ref{alg3:S-replace}, one previously evicted page from~$S$ replaces another page from $S$ in the cache (which is evicted). The former increases the number of evicted unmarked pages from~$S$ by one, while the latter maintains the number of evicted unmarked pages from~$S$.

Consider the first time,~$t$,
when there are no unmarked pages from~$S$ contained in the cache.
Until~$t$, whenever an unmarked page from~$S$ is loaded to the cache, it is marked and another
unmarked page from~$S$ is evicted. Therefore, there are precisely~$c$ unmarked pages from~$S$
which are not present in cache at time~$t$: the pages evicted at line~\ref{alg3:new-error} or the ones these have replaced at line~\ref{alg3:S-replace}.
Afterwards, no more evictions of~$0$-pages are made
and such pages are only loaded to the cache until it becomes full and a new phase starts.

To count evictions  made in line~\ref{alg3:S-replace}, we need to estimate the probability
of a requested unmarked page from~$S$ being missing from the cache.
We use an approach similar to the classical analysis of the algorithm \Mark.
Since it makes the situation only more costly for the algorithm, we can assume that all the evictions in line~\ref{alg3:new-error} are performed in the beginning of the phase and the evictions in line~\ref{alg3:S-replace}
are all performed afterwards.
When the~$j$th page from~$S$ is being marked, it is present in the cache with probability
$\frac{k-c-(j-1)}{k-(j-1)}$ (the numerator is the number of unmarked pages from~$S$ present in the cache at that moment and the denominator is the total number of unmarked pages in~$S$)
and the probability of a page fault is~$\frac{c}{k-(j-1)}$.
Therefore, the expected number of evictions in line~\ref{alg3:S-replace} until time~$t$
is
\[ 
    \sum_{j=1}^{k-c} \frac{c}{k-(j-1)} = c(H_k - H_c). 
\]
The total expected number of evictions of~$0$-pages during this phase is then
\[ c + c(H_k - H_c) \leq cH_k.\qedhere\]
\end{proof}

We observe that as a consequence of how the phases are defined, every page residing in the cache at the timepoint between two consecutive phases must have received its prediction during the phase that just ended. More formally,

\begin{observation}
\label{obs:alg3-eta0-overlap}
Let~$S(i)$ be the content of the cache when phase~$i-1$
ends and phase~$i$ starts.
Then all the pages in~$S(i)$ received their predictions during phase~$i-1$.
\end{observation}
\begin{proof}
This is a consequence of marking:
Every page requested during phase~$i-1$ received a new prediction.
The only pages from~$S(i-1)$ which did not, are the unmarked ones. Yet,
such pages are not present in the cache at the end of phase~$i-1$.
And all pages from~$S(i)\setminus S(i-1)$ must have been
requested and loaded during phase~$i-1$.
\end{proof}

We are now ready to analyze the $(\alpha,\beta,\gamma)$-competitiveness of the algorithm. It combines the previous results and uses the fact that evictions caused by a full cache can be charged to an erroneously predicted $0$-page, as trusting the predictions would require keeping more than~$k$ pages in cache. An additional factor is required in the dependency on~$\eta_0$ as a wrong prediction may impact both the current phase and the following one.

\ThmDiscardRandomizedUb*

\begin{proof}
Let us show that \Markzero is~$(1,2H_k, 1)$-competitive.
Consider a request sequence with optimum cost, \opt, during which \Markzero
performs~$m$ phases and
receives~$\eta_0$ incorrect predictions~$0$ and~$\eta_1$ incorrect predictions~$1$.
Let~$c_i$ denote the number of executions of line~\ref{alg3:new-error}
during phase~$i$.
Combining Observations~\ref{obs:alg3-1evict} and~\ref{obs:alg3-0evict},
the expected cost of \Markzero is at most
\[ OPT + \eta_1 + \sum_{i=1}^m c_i H_k.\]
It is enough to show that~$\sum_{i=1}^m c_i \leq 2\eta_0$ holds.

Consider the moment during phase~$i$ when line~\ref{alg3:new-error} is executed for the
$c_i$th time. At this moment, there are~$k$~$0$-pages in cache:
some of them belong to~$S(i)$, others were loaded during this phase.
Moreover, there are~$c_i-1$ unmarked pages from~$S(i)$ already evicted,
these are also~$0$-pages.
By Observation~\ref{obs:discard-error},
at least~$c_i$ of these pages must have an incorrect prediction of~$0$.
This prediction was received either during phase~$i-1$ (if it is an unmarked page from~$S(i)$),
or during phase~$i$ (all other cases).
Therefore, denoting~$\eta_0(i)$ the number of incorrect predictions 0 received during phase~$i$,
we have
\[ \sum_i c_i \leq \sum_{i=1}^m \big(\eta_0(i-1) + \eta_0(i)\big) \leq 2\eta_0, \]
which concludes the proof.
\end{proof}

\section{Algorithm with phase-predictions}
\label{sec:ub-phase}

In this section, we consider the phase-predictions setup and give a randomized algorithm, \MarkPredict. The idea of this algorithm is to follow the classical \Mark algorithm except that, instead of evicting a page uniformly at random among the set of unmarked pages, we select a 1-page if the cache contains one. We provide two analyses on the performance of \MarkPredict, which differ on the bound of the $\gamma$ parameter, the second bound providing an improvement for large values of $\eta_1$.

\begin{algorithm}
  \begin{algorithmic}[1]
    \State mark all pages in cache
    \For {$i = 1$ \textbf{to}~$n$}
      \If {$\page_i$ is not in cache}
        \If {all pages in cache are marked} \algorithmiccomment{Start of a new phase}
          \State unmark all pages
          
          \label{algline:phase-start}
        \EndIf
        \If {there is an unmarked~$1$-page}
          \State evict an unmarked~$1$-page chosen uniformly at random \label{algline:evict1}
        \Else
          \State evict an unmarked~$0$-page chosen uniformly at random  
          \label{algline:evict0}
        \EndIf
        \State bring~$\page_i$ into cache
      \EndIf
      \State mark~$\page_i$
    \EndFor
\end{algorithmic}
  \caption{\MarkPredict Eviction Strategy}
  \label{alg:markandpredict}
\end{algorithm}

The following observation is used in both proofs to estimate the value of $\eta_0$.
\begin{observation}
\label{obs:phase-error}
Consider a phase with $c$ new pages.
If $\ell\leq c$ of the $1$-pages present at the beginning of the phase were not requested during the phase,
then precisely
$z = c-\ell$ pages had incorrect $0$-predictions at the beginning
of the phase.
\end{observation}
\begin{proof}
Let $S$ denote the set of $0$-pages and $L$ the set of $1$-pages that were present at the beginning of the phase.
Denote by $z$ the number of pages in $S$ which were not requested during
this phase, so their predictions at the beginning of the phase were
incorrect.
During the phase $k = c + (|L|-\ell) + (|S|-z)$ distinct pages were requested.
Since $|S|+|L|=k$, we get $z = c-\ell$.
\end{proof}

We first provide an analysis which also holds if an arbitrary 1-page is evicted at Line~\ref{algline:evict1}, in a deterministic manner, say using LRU.

\ThmMarkRandA*

\begin{proof}
We use standard arguments for the competitive analysis of the randomized paging algorithm, MARK~\cite{FKLMSY91}, using terminology from the textbook by Borodin and El-Yaniv~\cite{BE98}. We first consider the case where all predictions are correct. Pages that arrive are always marked, so they are never evicted in the current~$k$-phase. Thus, the number of~$1$-pages that arrive in the current phase will be the number of~$1$-pages in cache at the beginning of the next phase. If all predictions in a phase are correct, the number of new pages in the  next~$k$-phase equals the number of~$1$-pages at the beginning of that phase, and the new pages will replace those~$1$-pages. There will be no faults on the~$0$-pages. Let~$c_i$ be the number of new pages in the~$i$th~$k$-phase and~$m$ be the total number of phases. Since the algorithm faults only on new pages, it faults~$\sum_{i=1}^m c_i$ times. We now turn to \opt. During the~$i-1$st and~$i$th~$k$-phases, at least~$k+c_i$ distinct pages have been requested. Since \opt cannot have had more than~$k$ of them in cache at the beginning of phase~$i-1$, it must have at least~$c_i$ faults in these two phases. Considering the even phases and the odd phases separately and taking the maximum, \opt must fault at least~$\frac{1}{2}\sum_{i=1}^m c_i$ times. This proves~$2$-consistency.

As long as $1$-pages are evicted, the faults are charged to \opt (if it is a correct prediction) or to~$\eta_1$ (if the prediction is incorrect). Since \opt is at least~$1/2$ times the total number of new pages, this gives a contribution of at most~$2 \opt + \eta_1$.

If the algorithm runs out of pages with~$1$-predictions to evict, there are only~$0$-pages from the previous phase remaining. For each new page processed after this point, there is an incorrect~$0$-prediction. Let~$z_i$ be the number of new pages causing a~$0$-page to be evicted in Phase~$i$. These new pages, causing evictions of pages with~$0$-predictions, arrive after the new pages that evicted pages with~$1$-predictions. 
The number of~$1$-pages present in the cache at the start of phase~$i$ is~$c_i-z_i$.

We can assume that all new pages arrive before any of the old pages, as this only increase the algorithm's cost.
When the first new page evicting a~$0$-page arrives, there are~$k-(c_i-z_i)$ pages from the previous phase still in cache and these~$k-(c_i-z_i)$ pages are all~$0$-pages.
When the first old page arrives, there are~$k-c_i$ pages from the previous phase in cache, so the arriving page has a probability of~$\frac{k-c_i}{k-(c_i-z_i)}$ of still being in the cache. 
 
Consider the probability that the~$j$th old page (in the order they arrive in this phase) is in cache the first time it is requested in the~$i$th phase. This probability  is~$\frac{k-c_i-(j-1))}{k-(c_i-z_i)-(j-1)}$, so the probability that there is a fault on it is~$\frac{z_i}{k-(c_i-z_i)-(j-1)}$. Hence the expected number of faults in Phase~$i$ due to incorrect~$0$-predictions is at most
\[z_i + \sum_{j=1}^{k-c_i}  \frac{z_i}{k-(c_i-z_i)-(j-1)} = z_i(1+H_{k-c_i+z_i}-H_{z_i})\leq z_i H_{k-c_i+z_i}.\]
By Observation~\ref{obs:phase-error}, the number of pages with incorrect $0$ prediction at the beginning
of the phase $i$ is $z_i$. So, this sum over all phases is at most~$H_k\eta_0$. 

The total number of faults is at most~$2 \opt  + H_k\eta_0 +\eta_1$.
\end{proof}

We provide another analysis of \MarkPredict which exploits the uniformly-random selection of an unmarked $1$-page to evict in line~\ref{algline:evict1}, and improves on the bound from Theorem~\ref{thm:markrand1} for larger values of~$\eta_1$.

\begin{lemma}
Consider a phase with~$c$ new pages, such that \MarkPredict
starts with~$\eta_0$ and~$\eta_1$ pages with incorrect predictions~$0$ and~$1$ in its cache.
The expected cost incurred by \MarkPredict is at most
\[ c \big(H_{\eta_1 + c} - H_c + 1\big) + H_k\,\eta_0. \]
\end{lemma}
\begin{proof}
Each phase starts at line~\ref{algline:phase-start} by unmarking all pages.
We denote~$L$ the set of 1-pages contained in the cache at this moment.
Note that any unmarked 1-page evicted at line~\ref{algline:evict1}
always belongs to~$L$.
We analyze two parts of the phase separately: (a) the first part when
there are still unmarked 1-pages in the cache and evictions are done
according to line~\ref{algline:evict1} and (b) when all unmarked 1-pages are evicted
and evictions are done by line~\ref{algline:evict0}.

\paragraph{\normalsize \sl Part (a)}
Marked pages are always in cache, therefore we only need to count page faults
when an unmarked page~$\page_i$ is requested.
Let~$c_a\leq c$ denote the number of new pages to arrive during the part (a).
Without loss of generality, we can assume that all of them arrive in the beginning.
There are three possibilities:
\begin{itemize}
    \item $\page_i$ is new:~$\MarkPredict$ incurs a cost of $1$.
    \item $\page_i$ is {not new and was} a~$0$-page {({\it i.e.}, the previous prediction on the page $\page_i$ is $0$)}: \MarkPredict incurs a cost of 0 (all such pages are in cache now)
    \item $\page_i$ is {not new and was} a~$1$-page:~\MarkPredict incurs a cost of~$\cost_i$ in expectation,
\end{itemize}
where~$\cost_i = c_a/(|L| - (j-1))$ if this was the~$j$-th page from~$L$ being marked.
This follows by an argument similar to the classical analysis of \Mark, as in the proof of Theorem~\ref{thm:markrand1}:
The probability that~$r_i$ is in cache is~$\frac{|L|-c_a-(j-1)}{|L|-(j-1)}$,
implying that the probability that~$r_i$ is missing from the cache is
$c_a/(|L|-(j-1))$.

Therefore, our expected cost during part (a) is at most
\[ c_a + \sum_{j=1}^{|L|-c_a} \frac{c_a}{|L|-(j-1)}
    = c_a (1 + H_{|L|} - H_{c_a}).
\]
At the end of the part (a), we have precisely~$c_a$ pages in $L$ that are no longer in the cache,
because part (b) starts only if there are more new pages in the phase than pages in $L$ that are never marked.
All the pages from~$L$ that are marked at the end of the phase had incorrect predictions, so we have
$\eta_1 \geq |L|-c_a$ implying~$|L| \leq \eta_1 + c_a$. Therefore,
our expected cost is at most
\[ c_a ( H_{\eta_1 + c_a} - H_{c_a} + 1)
    \leq c ( H_{\eta_1 + c} - H_c + 1).
\]

\paragraph{\normalsize\sl Part (b)}
This part never happens if~$c$ is the number of pages in~$L$ with correct prediction 1,
i.e., those left unmarked until the end of the phase.
If~$c$ is higher,
then there must have been some pages with incorrect prediction 0.
Without loss of generality, we can assume that all~$c-c_a$ new pages are requested
in the beginning of part (b).
Again, we only need to count page faults due to requests~$r_i$ where~$r_i$ is unmarked.
We have the following cases:
\begin{itemize}
    \item~$r_i$ is new:~\MarkPredict incurs a cost of 1,
    \item~$r_i\in L$:~\MarkPredict incurs a cost of 1 because all unmarked pages from~$L$
        are evicted by the end of part (a),
    \item~$r_i \notin L$:~\MarkPredict incurs a cost of~$\cost_i$.
\end{itemize}
Similar to the previous case, we have~$\cost_i = (c-c_a)/(k-|L| - (j-1))$ if
this is the~$j$th~$0$-page being marked, because the phase starts with~$k-|L|$~$0$-pages
in cache and they are not evicted during part (a).
By Observation~\ref{obs:phase-error},
each arrival of a new page and each request to
a further unmarked page in~$L$ increases~$\eta_0$ by 1. Moreover,
we have~$c-c_a \leq \eta_0$.
Therefore, our cost is at most
\[ \eta_0 + \sum_{j=1}^{k-|L|-(c-c_a)} \frac{\eta_0}{k-|L|-(j-1)}
    \leq \eta_0 (H_k - H_{\eta_0} + 1) \leq \eta_0 H_k. \qedhere
\]
\end{proof}

We next give a second upper bound on the $(\alpha,\beta,\gamma)$-competitiveness of \MarkPredict, which is stronger for large values of $\eta_1$.

\ThmMarkRandB*

\begin{proof}
Let~$c_i, \eta_1(i), \eta_0(i)$ be the numbers of new pages,~$1$-errors, and~$0$-errors in~$i$th phase,
respectively. We have~$\opt \geq \frac12 \sum_i c_i$.
Then, by the preceding lemma,
the cost of the algorithm is at most
\[ \sum_i \bigg(c_i \big(H_{\eta_1(i) + c_i} - H_{c_i} + 1\big) + H_k\,\eta_0(i) \bigg)
    \leq \sum_i c_i \bigg(\ln \big(\frac{\eta_1(i)}{c_i} + 1\big)+2 \bigg) + H_k\,\eta_0,
\]
where the sum is over all phases.
The inequality above holds because~$H_{\eta+c} - H_c\leq \log(\frac{\eta+c}{c}) + 1$.
By the concavity of logarithm, the worst case happens when~$\eta_1(i)/c_i$ is the same
in all the phases, i.e.,~$\eta_1(i)/c_i = 2\eta_1/\opt$ for all~$i$.
Therefore, the expected cost of \MarkPredict is at most
\[
    2\opt \left(\ln(\frac{2\eta_1}{\opt}+1) + 2\right) + H_k\eta_0\,
    \leq 2\opt + H_k\,\eta_0
    + \bigg(\ln(\frac{2\eta_1}{\opt}+1)+1\bigg)\frac{2\opt}{\eta_1}\eta_1. \qedhere
\]
\end{proof}

\section{Lower bounds}
\label{sec:lower-bounds}

In this section, we provide lower bounds on the possible values of $\alpha$, $\beta$ and $\gamma$ for
$(\alpha,\beta,\gamma)$-competitive algorithms, in both setups. {These bounds imply that the results of the previous two sections are essentially tight.}

We first consider deterministic algorithms.

\ThmDeterministicLowerBound*

\begin{proof}
Consider any deterministic paging algorithm \alg, and the following two paging problem instances on a universe of $k+1$ pages, each with $n$ requests, where $n>k$ can be arbitrarily large. When there are only $k+1$ pages used, the concept of $k$-phases for marking algorithms~\cite{T98} is used to show that \LFD faults on the first occurrence of each of the first $k$ pages requested in the first phase and on the first page in each phase after that, for a total of $\opt\leq k+\CEIL{\frac{n-k}{k}}$ faults. Ignoring the first and last phases, \LFD always evicts the only page not present in that phase, so correct predictions in the discard-predictions and phase-predictions setups are identical, with zeros for every request, except for the last occurrence of the page not requested in the next phase. (If the last phase contains fewer than $k$ different pages, there could be more than one correct $1$-prediction in the next to last phase, but one is sufficient. In the last phase, the correct predictions would all be zeros.)

In both instances, after $k$ requests, one to each of $k$ different pages, the unique page absent in the cache of \alg is always requested. This leads to a cost of $n$ for \alg, since it faults on all requests.

In the first instance, all predictions are $0$. Thus,  $\eta_0 \leq  \opt-k$ and $\eta_1=0$. Writing $\alg \leq \alpha\opt +\beta \eta_0+\gamma\eta_1$, we obtain that
\[
    n=\alg \leq \alpha\cdot \left(k+\CEIL{\frac{n-k}{k}}\right) +\beta \cdot \CEIL{\frac{n-k}{k}}.
\]
Taking the limit as $n$ goes to infinity, one must have 
\[
    \alpha+\beta \geq k.
\]

In the second instance, all predictions are $1$. Thus, $\eta_0 = 0$ and $\eta_1 \leq n- (s-k)$.  Writing $\alg \leq \alpha\opt +\beta \eta_0+\gamma\eta_1$, we obtain that
$$n= \alg \leq \alpha\cdot s +\gamma \cdot (n-(s-k)).$$
Since $\alpha\geq 1$, $\alpha\geq \gamma$. Then, $s\leq k+\CEIL{\frac{n-k}{k}}$ implies that
$$n=\alg \leq \alpha\cdot \left(k+\CEIL{\frac{n-k}{k}}\right) +
\gamma \cdot \left(n-\CEIL{\frac{n-k}{k}}\right).$$  
Taking the limit as $n$ goes to infinity, one must have 
\[
    \alpha+(k-1)\cdot\gamma \geq k. \qedhere
\]
\end{proof}

We now focus on randomized algorithms. The next result first considers a single instance with different predictions to exhibit two trade-offs on the possible competitive ratios, and the second trade-off is then improved using a different adversarial strategy.

\ThmRandLowerBound*

\begin{proof}
Consider any randomized paging algorithm \alg, and two paging problem instances on a universe of $k+1$ pages. In order to simplify the mathematical expressions, we assume that the instance starts with a full cache with predictions associated to each page. Since there is an additive constant in the definition of the competitive ratio, this does not affect the result.

In the first instance, for each request, one of the $k+1$ pages is chosen uniformly at random, with a prediction of $0$. This leads to an expected cost of approximately $n/(k+1)$ for \alg, as the probability that the requested page is the only one absent from the cache of \alg is $1/(k+1)$. 

The expected optimal cost is equal to the expected number of {\it $k$-phases} in the instance. The expected length of a phase is, by the Coupon Collector problem, $(k+1)H_{k+1}-1=(k+1)H_k$, where $H_i$ is the $i$-th harmonic number. So $\mathbb{E}[\opt] = n/((k+1)H_{k})$.

For the discard-predictions setup, this means that $\eta_0 = \opt$ and $\eta_1=0$ as each optimal eviction is equivalent to a prediction error.

For the phase-predictions setup, this also means that $\eta_0 = \opt$ and $\eta_1=0$ as each phase contains a single erroneous prediction, on the last request of the page not requested in the following phase.

Hence, we obtain that, for both setups,
\[
    \frac n {k+1}=\mathbb{E}[\alg] \leq \alpha\mathbb{E}[\opt] +\beta \mathbb{E}[\eta_0]+\gamma\eta_1 \leq (\alpha+\beta)\cdot \frac n{(k+1)H_{k}},
\]
so
\[
    \alpha+\beta \geq H_{k}.
\]

We note below that replacing the predictions from 0 to 1 does not lead to the target bound.
Indeed, consider an instance such that, at each round, one of the $k+1$ pages is requested at random, with a prediction of $1$. This again leads to an expected cost of $n/(k+1)$ for \alg and $n/((k+1)H_{k})$ for \opt. This means that $\eta_1 \leq n$ and $\eta_0=0$ for both setups. 
Hence, we obtain that, for both setups,
\[
    \frac n {k+1}=\mathbb{E}[\alg] \leq \alpha\mathbb{E}[\opt] +\beta \eta_0+\gamma\mathbb{E}[\eta_1] \leq (\alpha+(k+1)H_{k}\cdot\gamma)\cdot \frac n{(k+1)H_{k}},
\]
so
\[
    \alpha+(k+1)H_{k}\cdot\gamma \geq H_{k}.
\]

In order to improve this bound, we keep a universe of $k+1$ pages and build an instance phase by phase, based on the cache $C$ of an optimal solution before the start of the phase. The first request is the page $p_0$ not in $C$. Then, we consider a uniformly random permutation $\sigma_1,\dots,\sigma_{k-1}$ of $k-1$ among the $k$ elements of $C$. The phase will then be described as a composition of \emph{blocks} of requests, where the $i$th block contains $i+1$ page requests: $p_0$ and the $\sigma_j$ for $j\leq i$.
For instance, if the permutation is $(a,b,c,d,e)$, the blocks will be:
\[
    p_0a,p_0ab,p_0abc,p_0abcd,p_0abcde. 
\]
Each block is furthermore repeated several times before requesting the next block to ensure that the cache of any sensible algorithm contains the pages inside a block afterwards.

We now compute a lower bound on the expected cost of any algorithm on such a sequence. Before the first block, $p_0$ is contained in the cache, so the probability that requesting $a$ incurs a cache miss is $1/k$, as, except $p_0$, one of the $k$ other pages in the universe incurs a cache miss. Similarly, the probability that the second block incurs a cache miss is $1/(k-1)$, and the total expected number of cache misses after the last block is $H_k-1$. We now notice that we can also charge one eviction for $p_0$ at the start of the phase. Indeed, after the previous phase was finished, the algorithm cache must contain the $k$ pages of the last block, to avoid suffering too many evictions. Therefore, its cache at the start of the phase matches the one of \opt so does not contain $p_0$. So the algorithm cost is at least $H_k$ per phase while $\opt$'s is one per phase.

We now describe predictions: all requests come with a prediction $0$ except the last iteration of the last block where all $k$ requests have a prediction $1$. 

For the discard-predictions setup, the zero-prediction are correct as these pages are requested again in the same phase, and $k-1$ one-predictions are wrong on the last iteration as only a single page should be evicted, so $\eta_0=0$ and $\eta_1=k-1$ per phase.

For the phase-predictions setup, only the last iteration counts towards the error, and a single page will not appear in the next phase so we also have $\eta_0=0$ and $\eta_1=k-1$ per phase.
Therefore, generalizing to all phases, we have
\[ H_k=\mathbb{E}[\alg] \leq \alpha\mathbb{E}[\opt] +\beta \eta_0+\gamma\mathbb{E}[\eta_1] \leq \alpha+(k-1)\cdot\gamma. \qedhere \]
\end{proof}

The previous result shows a subconstant lower bound on $\gamma$ ($\approx \frac 1k\ln k$) for a logarithmic value of $\alpha$ (up to $O(\log k)$), and we complement it by showing that $\gamma$ is lower bounded by a constant if we want to achieve $\alpha=1$.

\ThmRandLowerBoundB*

\begin{proof}
We consider a universe of $k+1$ pages. We construct an instance composed of $m$ rounds of $k-1$ requests, $m$ being a large integer. 
At the start of each round, request the page 1, 2 or 3 with equal probability associated to a prediction 1. Then, all pages from 4 to $k+1$ are requested with a prediction 0.

An optimal algorithm never evicts the pages 4 to $k+1$ and needs to evict a single page per phase, where phases are defined as for marking algorithms. Any online algorithm has a probability at least $1/3$ to perform an eviction at each round: either one page among $\{1,2,3\}$ is not in the cache at the start of the round, or another page is absent which enforces an eviction.

The expected number of rounds in a phase is equal to the expected length of a phase of a uniformly random request sequence over 3 pages and $k=2$, which is $3H_2=4.5$. So $\mathbb{E}[\opt] = m/4.5$.

We now focus on the prediction errors.
First, note that $\eta_0=0$ in both setups: the pages predicted 0 are requested in every phase and should never be evicted by an optimal algorithm.

Then, we have $\mathbb E[\eta_1] = m-\mathbb E[\opt] = 3.5m/4.5 $ in the discard-predictions setup. Indeed, the pages 1, 2 and 3 combined are predicted 1 a total of $m$ times, are never predicted 0 and \opt only evicts these three pages. So there is an error when such a page is not evicted by $\opt$ before its subsequent request. 

In the phase-predictions setup, there is one error per phase, for the last prediction of the unique page among $\{1,2,3\}$ which is both requested in that phase but not the following one. Therefore, $\mathbb{E}[\eta_1] = m/3H_2 = m/4.5 $.

Therefore, we have:
\begin{align*}
\mathbb{E}[\alg] &\leq 1 \cdot \mathbb{E}[\opt] + \gamma \mathbb{E}[\eta_1]\\
\gamma &\geq \frac{\mathbb{E}[\alg]-\mathbb{E}[\opt]}{\mathbb{E}[\eta_1]}\\
\gamma &\geq \frac m {\mathbb{E}[\eta_1]}  \cdot \left({\frac{1}{3}-\frac{1}{3H_2}}\right)
\geq \frac m {\mathbb{E}[\eta_1]}  \cdot \left({\frac{1.5-1}{4.5}}\right) \geq \frac m{9\mathbb{E}[\eta_1]}
\end{align*}
So, in the discard-predictions setup, we get $\mathbb{E}[\alg] \geq 1/7$ and for the phase-predictions setup, we obtain $\mathbb{E}[\alg] \geq 1/2$.
\end{proof} 

\bibliography{refs}
\bibliographystyle{plain}

\end{document}